\documentclass{llncs}
\usepackage{graphicx}
\usepackage{amssymb}
\usepackage{amsmath}
\usepackage{bm}
\usepackage[small]{caption}
\usepackage[english]{babel}
\usepackage{stmaryrd}

\DeclareMathOperator*{\argmax}{arg\,max}

\newcommand{\vsim}{\mathrel{\scalebox{1}[1.5]{$\shortmid$}\mkern-3.1mu\raisebox{0.1ex}{$\sim$}}}
\newcommand{\vapprox}{\mathrel{\scalebox{1}[1.5]{$\shortmid$}\mkern-3.1mu\raisebox{0.1ex}{$\approx$}}}

\usepackage[whole]{bxcjkjatype}

\usepackage{times}

\title{Bayesian Entailment Hypothesis: How Brains Implement Monotonic and Non-monotonic Reasoning\thanks{This paper was submitted to IJCAI 2020 and rejected.}}
\author{Hiroyuki Kido}
\institute{Cardiff University\\
\email{KidoH@cardiff.ac.uk}}
\begin{document}
\begin{sloppypar}
\maketitle
\begin{abstract}
Recent success of Bayesian methods in neuroscience and artificial intelligence makes researchers conceive the hypothesis that the brain is a Bayesian machine. Since logic, as the laws of thought, is a product and a practice of our human brains, it is natural to think that there is a Bayesian algorithm and data-structure for entailment. This paper gives a Bayesian account of entailment and characterizes its abstract inferential properties. The Bayesian entailment is shown to be a monotonic consequence relation in an extreme case. In general, it is a non-monotonic consequence relation satisfying Classical cautious monotony and Classical cut we introduce to reconcile existing conflicting views in the field. The preferential entailment, which is a representative non-monotonic consequence relation, is shown to correspond to a maximum a posteriori entailment. It is an approximation of the Bayesian entailment. We derive it based on the fact that maximum a posteriori estimation is an approximation of Bayesian estimation. We finally discuss merits of our proposals in terms of encoding preferences on defaults, handling change and contradiction, and modeling human entailment.
\end{abstract}

\section{Introduction}
Bayes' theorem is a simple mathematical equation published in 1763. Today, it plays an important role in various fields such as AI, neuroscience, cognitive science, statistical physics and bioinformatics. It lays the foundation of most modern AI systems including self-driving cars, machine language translation, speech recognition and medical diagnosis \cite{russell:09}. Recent studies of neuroscience, e.g., \cite{lee:03,knill:04,george:05,ichisugi:07,chikkerur:10,colombo:12,funamizu:16}, empirically show that Bayesian methods explain several functions of the cerebral cortex. It is the outer portion of the brain in charge of higher-order functions such as perception, memory, emotion and thought. These successes of Bayesian methods make researchers conceive the Bayesian brain hypothesis that brain is a Bayesian machine \cite{friston:12,sanborn:16}.
%
%
\par
If the Bayesian brain hypothesis is true then it is natural to think that there is a Bayesian algorithm and data-structure for logical reasoning. This is because logic, as the laws of thought, is a product and a practice of our human brains. Such Bayesian account of logical reasoning is important. First, it has a potential to be a mathematical model to explain how the brain implements logical reasoning. Second, it theoretically supports the Bayesian brain hypothesis in terms of logic. Third. it gives an opportunity and a way to critically assess the existing formalisms of logical reasoning. Nevertheless, few research has focused on reformulating logical reasoning in terms of the Bayesian perspective (see Section \ref{sec:dis} for discussion).
\par
In this paper, we begin by assuming the posterior distribution over valuation functions, denoted by $v$. The probability of each valuation function represents how much the state of the world specified by the valuation function is natural, normal or typical. We then assume a causal relation from valuation functions to each sentence, denoted by $\alpha$. Under the assumptions, the probability that $\alpha$ is true, denoted by $p(\alpha)$, will be shown to have
%
\begin{eqnarray*}
p(\alpha)=\sum_{v}p(\alpha,v)=\sum_{v}p(\alpha|v)p(v).
\end{eqnarray*}
That is, the probability of any sentence is not primitive but dependent on all valuation functions probabilistically distributed. Given a set $\Delta$ of sentences with the same assumptions, we will show to have
%
%
\begin{eqnarray*}
p(\alpha|\Delta)=\sum_{v}p(\alpha|v)p(v|\Delta).
\end{eqnarray*}
This equation is known as Bayesian learning \cite{russell:09}. Intuitively speaking, $\Delta$ updates the probability distribution over valuation functions, i.e., $p(v)$ for all $v$, and then the updated distribution is used to predict the truth of $\alpha$. We define a Bayesian entailment, denoted by $\Delta\vapprox_{\omega}\alpha$, using $p(\alpha|\Delta)\geq\omega$, as usual. 
\par
We derive several important facts from the idea. The Bayesian entailment is shown to be a monotonic consequence relation when $\omega=1$. In general, it is a non-monotonic consequence relation satisfying Classical cautious monotony and Classical cut we introduce to reconcile existing conflicting views \cite{gabbay:85,brewka:97} in the field. The preferential entailment \cite{shoham:87}, which is a representative non-monotonic consequence relation, is shown to correspond to a maximum a posteriori entailment. It is an approximation of the Bayesian entailment. It is derived from the fact that maximum a posteriori estimation is an approximation of Bayesian estimation. These results imply that both monotonic and non-monotonic consequence relations can be seen as Bayesian learning with a fixed probability threshold.
\par
This paper is orgranized as follows. Section 2 gives a probabilistic model for a Bayesian entailment. For the sake of correctness of the Bayesian entailment, Section 3 discusses its inferential properties in terms of monotonic and non-monotonic consequence relations. In Section 4, we discuss merits of our proposals in terms of encoding preferences on defaults, handling change and contradiction, and modeling human entailment.
%
\section{Bayesian Entailment}
Let ${\cal L}$, ${\cal P}$ and $v$ respectively denote the propositional language, the set of all propositional symbols in ${\cal L}$, and a valuation function, $v:{\cal P}\rightarrow\{0,1\}$, where $0$ and $1$ mean the truth values, \textit{false} and \textit{true}, respectively. To handle uncertainty of states of the world, we assume that valuation functions are probabilistically distributed. Let $V$ denote a random variable over valuation functions. $p(V=v_{i})$ denotes the probability of valuation function $v_{i}$. It reflects the probability of the state of the world specified by $v_{i}$. Given two valuation functions $v_{1}$ and $v_{2}$, $p(V=v_{1})>p(V=v_{2})$ represents that the state of the world specified by $v_{1}$ is more natural/typical/normal than that of $v_{2}$. When the cardinality of ${\cal P}$ is $n$, there are $2^{n}$ possible states of the world. Thus, there are $2^{n}$ possible valuation functions. It is the case that $0\leq p(v_{i})\leq 1$, for all $i$ such that $1\leq i\leq 2^{n}$, and $\sum_{i=1}^{2^{n}}p(V=v_{i})=1$.
\par
We assume that every propositional sentence is a random variable that has a truth value either $0$ or $1$. For all $\alpha\in{\cal L}$, $p(\alpha=1)$ represents the probability that $\alpha$ is \textit{true} and $p(\alpha=0)$ that $\alpha$ is \textit{false}. We assume that $\llbracket\alpha\rrbracket$ denotes the set of all valuation functions in which $\alpha$ is true and $\llbracket\alpha\rrbracket_{v}$ denotes the truth value under valuation function $v$.
%

\begin{definition}[Interpretation]\label{def:int}
Let $\alpha$ be a propositional sentence and $V$ be a valuation function. The conditional probability distribution over $\alpha$ given $V$ is given as follows.
\begin{eqnarray*}
p(\alpha=1|V)&=&\llbracket\alpha\rrbracket_{V}\\
p(\alpha=0|V)&=&1-\llbracket\alpha\rrbracket_{V}
\end{eqnarray*}
\end{definition}
From the viewpoint of logic, it is natural to assume that a truth value of a sentence is caused only by the valuation functions. Thus, $p(\alpha)$ is given by
\begin{eqnarray*}
p(\alpha)=\sum_{v_{i}}p(\alpha,V=v_{i})=\sum_{v_{i}}p(\alpha|V=v_{i})p(V=v_{i}).
\end{eqnarray*}
\begin{example}
Suppose two propositional symbols $a$ and $b$. The left table in Table \ref{ex.Int} shows all of the $2^{2}=4$ possible valuation functions and their probability distribution. The right table in Table \ref{ex.Int} shows $p(a\lor\lnot b|V)$. Let us assume the following probability distribution over valuation functions.
\begin{eqnarray*}
p(V)&=&(p(V=v_{1}),p(V=v_{2}),p(V=v_{3}),p(V=v_{4}))\\
&=&(0.5, 0.2, 0, 0.3)
\end{eqnarray*}
Now, $p(a\lor\lnot b=1)$ and $p(a\lor\lnot b=0)$ are given as follows.
\begin{eqnarray*}
p(a\lor\lnot b=1)&=&\sum_{i=1}^{4}p(a\lor\lnot b=1|V=v_{i})p(V=v_{i})\\
&=&\sum_{i=1}^{4}\llbracket a\lor\lnot b\rrbracket_{V=v_{i}}p(V=v_{i})\\
&=&p(V=v_{1})+p(V=v_{3})+p(V=v_{4})\\
&=&0.8\\
p(a\lor\lnot b=0)&=&1-p(a\lor\lnot b=1)\\
&=&0.2
\end{eqnarray*}
\end{example}
\begin{table}[t]
\caption{The left table shows all valuation functions and their distribution. The right table shows $p(a\lor\lnot b|V)$.}
\label{ex.Int}
\begin{center}
\begin{tabular}{cc}
\begin{minipage}{0.5\hsize}
\begin{center}
\begin{tabular}{c|c|cc}
& $p(V)$ & $a$ & $b$\\\hline
$v_{1}$ & $0.5$ & $0$ & $0$\\
$v_{2}$ & $0.2$ & $0$ & $1$\\
$v_{3}$ & $0$ & $1$ & $0$\\
$v_{4}$ & $0.3$ & $1$ & $1$
\end{tabular}
\end{center}
\end{minipage}
\begin{minipage}{0.5\hsize}
\begin{center}
\begin{tabular}{c|cc}
& $a\lor\lnot b=0$ & $a\lor\lnot b=1$\\\hline
$V=v_{1}$ & $0$ & $1$\\
$V=v_{2}$ & $1$ & $0$ \\
$V=v_{3}$ & $0$ & $1$ \\
$V=v_{4}$ & $0$ & $1$
\end{tabular}
\end{center}
\end{minipage}
\end{tabular}
\end{center}
\end{table}
\par
Definition \ref{def:int} implies that the probability of the truth of a sentence is not primitive, but dependent on the valuation functions. Therefore, we need to guarantee that probabilities on sentences satisfy the Kolmogorov axioms.
%
\begin{proposition}\label{kolmogorov}
The following expressions hold, for all formulae $\alpha,\beta\in{\cal L}$.
\begin{enumerate}
\item $0\leq p(\alpha=i)$, for all $i$.
\item $\sum_{i}p(\alpha=i)=1$.
\item $p(\alpha\lor\beta=i)=p(\alpha=i)+p(\beta=i)-p(\alpha\land\beta=i)$, for all $i$.
\end{enumerate}
\end{proposition}
\begin{proof}
Since any sentence takes a truth value either $0$ or $1$, it is sufficient for (1) and (2) to show $p(\alpha=0)+p(\alpha=1)=1$ and $0\leq p(\alpha=0),p(\alpha=1)\leq 1$. The following expressions hold.
\begin{eqnarray*}
p(\alpha=0)&=&\sum_{v}p(\alpha=0|v)p(v)=\sum_{v}(1-\llbracket\alpha\rrbracket_{v})p(v)\\
p(\alpha=1)&=&\sum_{v}p(\alpha=1|v)p(v)=\sum_{v}\llbracket\alpha\rrbracket_{v}p(v)
\end{eqnarray*}
Here, we have abbreviated $p(V=v)$ to $p(v)$ and $\llbracket\alpha\rrbracket_{V=v}$ to $\llbracket\alpha\rrbracket_{v}$. (1) is true because $0\leq p(v)$ holds, for all $v$. (2) is true because $p(\alpha=0)+p(\alpha=1)=\sum_{v}p(v)=1$ holds. (3) can be developed as follows, where the first expression comes when $i=0$ and the second when $i=1$. 
\begin{eqnarray*}
\sum_{v}p(v)(1-\llbracket\alpha\lor\beta\rrbracket_{v})&=&\sum_{v}p(v)(1-\{\llbracket\alpha\rrbracket_{v}+\llbracket\beta\rrbracket_{v}-\llbracket\alpha\land\beta\rrbracket_{v}\})\\
\sum_{v}p(v)\llbracket\alpha\lor\beta\rrbracket_{v}&=&\sum_{v}p(v)\{\llbracket\alpha\rrbracket_{v}+\llbracket\beta\rrbracket_{v}-\llbracket\alpha\land\beta\rrbracket_{v}\}
\end{eqnarray*}
There are four possible cases. If $\llbracket\alpha\rrbracket_{v}=\llbracket\beta\rrbracket_{v}=0$ then the expression in the bracket of the right expressions turn out to be $0(=0+0-0)$, if $\llbracket\alpha\rrbracket_{v}=0$ and $\llbracket\beta\rrbracket_{v}=1$ then $1(=0+1-0)$, if $\llbracket\alpha\rrbracket_{v}=1$ and $\llbracket\beta\rrbracket_{v}=0$ then $1(=1+0-0)$, and if $\llbracket\alpha\rrbracket_{v}=\llbracket\beta\rrbracket_{v}=1$ then $1(=1+1-1)$. All the results are consistent with $\llbracket\alpha\lor\beta\rrbracket_{v}$.
\end{proof}
\par
%
\begin{proposition}\label{negation}
$p(\alpha=0)=p(\neg\alpha=1)$ holds, for any $\alpha\in{\cal L}$.
\end{proposition}
\begin{proof}
It is true that $p(\lnot\alpha=1)=\sum_{v}\llbracket\lnot\alpha\rrbracket_{v} p(v)=\sum_{v}(1-\llbracket\alpha\rrbracket_{v})p(v)=p(\alpha=0)$.
\end{proof}
In what follows, we thus replace $p(\alpha=0)$ by $p(\lnot\alpha=1)$ and then abbreviate $p(\lnot\alpha=1)$ to $p(\lnot\alpha)$, for all sentences $\alpha\in{\cal L}$.
%
%
\par
Dependency among random variables is shown in Figure \ref{dependency2} using a Bayesian network, a directed acyclic graphical model. A sentence $\alpha$ has a directed edge only from a valuation function $V$. It represents that valuation functions are the direct causes of truth values of sentences. The dependency between $V$ and another sentence $\beta_{i}$ is the same as $\alpha$. Only $\beta_{i}$ is coloured grey. It means that $\beta_{i}$ is assumed to be observed, which is in contrast to the other nodes assumed to be predicted or estimated. The box surrounding $\beta_{i}$ is a plate. It represents that there are $N$ sentences $\beta_{1}$, $\beta_{2}$, ... and $\beta_{N}$ to which there is a directed edge from $V$. Given the dependency, the conditional probability of $\alpha$ given $\beta_{1}$, $\beta_{2}$, ... and $\beta_{N}$ is given as follows.
%
%
\begin{eqnarray*}
p(\alpha |\beta_{1},\beta_{2},\cdots,\beta_{N})&=&\frac{p(\alpha,\beta_{1},\beta_{2},\cdots,\beta_{N})}{p(\beta_{1},\beta_{2},\cdots,\beta_{N})}=\frac{\sum_{v}p(v)p(\alpha|v)\prod_{i=1}^{N}p(\beta_{i}|v)}{\sum_{v}p(v)\prod_{i=1}^{N}p(\beta_{i}|v)}\\
&=&\frac{\sum_{v}p(v)\llbracket\alpha\rrbracket_{v}\prod_{i=1}^{N}\llbracket\beta_{i}\rrbracket_{v}}{\sum_{v}p(v)\prod_{i=1}^{N}\llbracket\beta_{i}\rrbracket_{v}}
=\frac{\sum_{v}p(v)\llbracket\alpha\rrbracket_{v}\llbracket\beta_{1},\beta_{2},...,\beta_{N}\rrbracket_{v}}{\sum_{v}p(v)\llbracket\beta_{1},\beta_{2},...,\beta_{N}\rrbracket_{v}}\\
&=&\frac{\sum_{v\in\llbracket\alpha,\beta_{1},\beta_{2},...,\beta_{N}\rrbracket}p(v)}{\sum_{v\in\llbracket\beta_{1},\beta_{2},...,\beta_{N}\rrbracket}p(v)}
\end{eqnarray*} 
%
%
\begin{example}[Continued]
$p(\lnot a|a\lor\lnot b, \lnot a\lor b)$ is given as follows.
\begin{eqnarray*}
p(\lnot a|a\lor\lnot b, \lnot a\lor b)&=&\frac{\sum_{v}p(v)\llbracket\lnot a\rrbracket_{v}\llbracket a\lor\lnot b, \lnot a\lor b \rrbracket_{v}}{\sum_{v}p(v)\llbracket a\lor\lnot b, \lnot a\lor b\rrbracket_{v}}\\
&=&\frac{p(v_{1})}{p(v_{1})+p(v_{4})}=\frac{0.5}{0.8}=0.625
\end{eqnarray*}
\end{example}

\par
Now, we want to investigate logical properties of $p(\alpha |\beta_{1},\beta_{2},...,\beta_{N})$. We thus define a consequence relation between $\{\beta_{1},\beta_{2},...,\beta_{N}\}$ and $\alpha$.
%
%
\begin{definition}[Bayesian entailment]\label{def:BE}
Let $\alpha\in{\cal L}$ be a sentence, $\Delta\subseteq{\cal L}$ be a set of sentences, and $\omega\in[0,1]$ be a probability. $\alpha$ is a Bayesian entailment of $\Delta$ with $\omega$, denoted by $\Delta\vapprox_{\omega}\alpha$, if $p(\alpha|\Delta)\geq\omega$ or $p(\Delta)=0$.
%
%
\end{definition}
Condition $p(\Delta)=0$ guarantees that $\alpha$ is a Bayesian entailment of $\Delta$ when $p(\alpha|\Delta)$ is undefined due to division by zero. It happens when $\Delta$ has no models, i.e., $\llbracket\Delta\rrbracket=\emptyset$, or zero probability, i.e., $p(v)=0$, for all $v\in\llbracket\Delta\rrbracket$. $\vapprox_{\omega}\alpha$ is a special case of Definition \ref{def:BE}. It holds when $p(\alpha)\geq\omega$. We call Definition \ref{def:BE} Bayesian entailment because $p(\alpha|\Delta)$ can be developed as follows.
%
\begin{eqnarray*}
p(\alpha |\Delta)=\frac{\sum_{v}p(v,\alpha,\Delta)}{p(\Delta)}=\frac{\sum_{v}p(\alpha|v)p(v|\Delta)p(\Delta)}{p(\Delta)}=\sum_{v}p(\alpha|v)p(v|\Delta)
\end{eqnarray*} 
The expression is often called \emph{Bayesian learning} where $\Delta$ updates the distribution over valuation functions, i.e., $p(V)$, and truth values of $\alpha$ is predicted using the updated distribution. Therefore, the Bayesian entailment allows us to see consequences in logic are predictions in Bayesian learning.
\begin{figure}[t]
\begin{center}
 \includegraphics[scale=0.4]{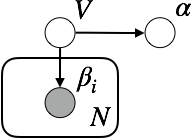}
  \caption{Dependency between the random variable $V$ of a valuation function and the random variables $\alpha$ and $\beta_{i}$ of propositional sentences.}
  \label{dependency2}
\end{center}
\end{figure}

\section{Correctness}
%
This section aims to show correctness of the Bayesian entailment. We first prove that a natural restriction of the Bayesian entailment is a classical consequence relation. We then prove that a weaker restriction of the Bayesian entailment can be seen as a non-monotonic consequence relation.
\subsection{Propositional Entailment}
Recall that propositional entailment $\Delta\models\alpha$ is defined as follows: For all valuation functions $v$, if $\Delta$ is true in $v$ then $\alpha$ is true in $v$. The Bayesian entailment $\vapprox_{1}$ works in a similar way as the propositional entailment. The only difference is that the Bayesian entailment ignores valuation functions with zero probability. The valuation functions with zero probability represent impossible states of the world.
\begin{theorem}\label{thrm:1}
Let $\alpha\in{\cal L}$ be a sentence and $\Delta\subseteq{\cal L}$ be a set of sentences. $\Delta\vapprox_{1}\alpha$ holds if and only if, for all valuation functions $v$ such that $p(v)\neq 0$, if $\Delta$ is true in $v$ then $\alpha$ is true in $v$.
\end{theorem}
\begin{proof}
We show that $\Delta\not\vapprox_{1}\alpha$ holds if and only if there is a valuation function $v$ such that $p(v)\neq 0$ holds, $\Delta$ is true in $v$, and $\alpha$ is false in $v$. From Definition \ref{def:BE}, $\Delta\not\vapprox_{1}\alpha$ holds if and only if $p(\Delta)\neq 0$ and $p(\alpha|\Delta)\neq 1$ hold. From Definition \ref{def:int}, $p(\Delta)\neq 0$ holds if and only if there is a valuation function $v^{*}$ such that $p(v^{*})\neq 0$ holds and $\Delta$ is true in $v^{*}$. ($\Leftarrow$) This has been proven so far. ($\Rightarrow$) $p(\Delta)\neq 0$ holds due to $v^{*}$. Since $p(\alpha|\Delta)=\frac{\sum_{v}p(v)\llbracket\alpha\rrbracket_{v}\llbracket\Delta\rrbracket_{v}}{\sum_{v}p(v)\llbracket\Delta\rrbracket_{v}}\neq 1$, there is $v\in\llbracket\Delta\rrbracket\setminus\llbracket\alpha\rrbracket$ such that $p(v)\neq 0$. $\Delta$ is true in $v$ and $\alpha$ is false in $v$.
\end{proof}
%
When no valuation functions take zero probability, consequences with the Bayesian entailment with probability one coincide with the propositional entailment.
%
\begin{proposition}\label{prop:1}
Let $\alpha\in{\cal L}$ be a sentence, $\Delta\subseteq{\cal L}$ be a set of sentences. If there is no valuation functions $v$ such that $p(v)= 0$ then $\Delta\vapprox_{1}\alpha$ holds if and only if $\Delta\models\alpha$ holds.
\end{proposition}
\begin{proof}
The definition of the propositional entailment is equivalent to Theorem \ref{thrm:1} under the non-zero assumption.
\end{proof}
No preference is given to valuation functions in the propositional entailment. It corresponds to the Bayesian entailment with the uniform assumption, i.e., $p(v_{1})=p(v_{2})=\cdots=p(v_{N})$. Note that Theorem \ref{thrm:1} and Proposition \ref{prop:1} hold even when the probability distribution over valuation functions is uniform. It is because the uniform assumption is stricter than the non-zero assumption, i.e., $p(v)\neq 0$. The following holds when no assumption is imposed on the probability distribution over valuation functions, 
%
%
%
\begin{proposition}\label{prop:2}
Let $\alpha\in{\cal L}$ be a sentence and $\Delta\subseteq{\cal L}$ be a set of sentences. If $\Delta\vDash\alpha$ holds then $\Delta\vapprox_{1}\alpha$ holds, but not vice versa.
\end{proposition}
\begin{proof}
This is obvious from Theorem \ref{thrm:1}.
\end{proof}
Since an entailment is defined between a set of sentences and a sentence, both $\models$ and $\vapprox_{\omega}$ are subsets of $Pow({\cal L})\times {\cal L}$. Here, $Pow({\cal L})$ denotes the power set of the language ${\cal L}$. Proposition \ref{prop:2} states that $\models$ is a subset of $\vapprox_{1}$, i.e., $\models~\subseteq~\vapprox_{1}$. It represents that  the propositional entailment is more cautious in entailment than the Bayesian entailment. Moreover, it is obvious from Definition \ref{def:BE} that $\vapprox_{\omega_{1}}~\subseteq~\vapprox_{\omega_{2}}$ holds if $\omega_{2}\leq\omega_{1}$ holds.
\par
In propositional logic, a sentence is said to be a tautology if it is true in all of the valuation functions. A sentence $\alpha$ is thus a tautology if and only if $\models\alpha$ holds. In the Bayesian entailment, $\vapprox_{1}\alpha$ holds if and only if $\alpha$ is true in all of the valuation functions with a non-zero probability.
\begin{proposition}\label{prop:3}
Let $\alpha\in{\cal L}$ be a sentence. $\vapprox_{1}\alpha$ holds if and only if, for all valuation functions $v$ such that $p(v)\neq 0$, $\alpha$ is true in $v$.
\end{proposition}
\begin{proof}
$p(\alpha)=\sum_{v}\llbracket\alpha\rrbracket_{v}p(v)=\sum_{v\in\llbracket\alpha\rrbracket}p(v)$. Since $\sum_{v}p(v)= 1$, $\sum_{v\in\llbracket\alpha\rrbracket}p(v)\neq 1$ holds if and only if there is $v^{*}$ such that $v^{*}\notin\llbracket\alpha\rrbracket$ and $p(v^{*})\neq0$.
\end{proof}
%
\begin{example}[Continued]
$\vapprox_{1}\lnot a\lor b$ holds because we have
\begin{eqnarray*}
p(\lnot a\lor b)=\sum_{v\in\llbracket\lnot a\lor b\rrbracket}p(v)=p(v_{1})+p(v_{2})+p(v_{4})=1.
\end{eqnarray*}
\end{example}
\subsection{Monotonic Consequence Relation}
We investigate inferential properties of the Bayesian entailment in terms of the monotonic consequence relation. It is known that any monotonic consequence relation can be characterised by the three properties: \textit{Reflexivity}, \textit{Monotony} and \textit{Cut}. Let $\vdash\subseteq Pow({\cal L})\times{\cal L}$ denote a consequence relation on ${\cal L}$. Those properties are defined as follows, where $\alpha, \beta\in{\cal L}$ and $\Delta\subseteq{\cal L}$.
%
\begin{itemize}
\item Reflexivity: $\forall\Delta\forall\alpha$, $\Delta,\alpha\vdash \alpha$
\item Monotony: $\forall\Delta\forall\alpha\forall\beta$, if $\Delta\vdash\alpha$ then $\Delta,\beta\vdash\alpha$
\item Cut: $\forall\Delta\forall\alpha\forall\beta$, if $\Delta\vdash\beta$ and $\Delta,\beta\vdash\alpha$ then $\Delta\vdash\alpha$
\end{itemize}
Reflexivity states that $\alpha$ is a consequence of any set with $\alpha$. Monotony states that if $\alpha$ is a consequence of $\Delta$ then it is a consequence of any superset of $\Delta$ as well. Cut states that an addition of any consequence of $\Delta$ to $\Delta$ does not reduce any consequence of $\Delta$. The Bayesian entailment $\vapprox_{1}$ is classical in the sense that it satisfies all of the properties.
%
\begin{theorem}\label{thrm:2}
The Bayesian entailment $\vapprox_{1}$ satisfies Reflexivity, Monotony and Cut.
\end{theorem}
\begin{proof}
(Reflexivity) It is true because $\models~\subseteq~\vapprox_{1}$ holds. (Monotony) Since $\Delta\vapprox_{1}\alpha$ holds, $\llbracket\Delta\rrbracket\subseteq\llbracket\alpha\rrbracket$ or $p(v)=0$ holds, for all $v\in\llbracket\Delta\rrbracket\setminus\llbracket\alpha\rrbracket$. For all $v\notin\llbracket\Delta\rrbracket\setminus\llbracket\alpha\rrbracket$, it is thus true that if $v\in\llbracket\Delta\rrbracket$ holds then $v\in\llbracket\alpha\rrbracket$. Therefore,
\begin{eqnarray*}
p(\alpha|\Delta,\beta)=\frac{p(\alpha,\Delta,\beta)}{p(\Delta,\beta)}=\frac{\sum_{v}\llbracket\alpha\rrbracket_{v}\llbracket\Delta\rrbracket_{v}\llbracket\beta\rrbracket_{v}p(v)}{\sum_{v}\llbracket\Delta\rrbracket_{v}\llbracket\beta\rrbracket_{v}p(v)}=\frac{\sum_{v\notin\llbracket\Delta\rrbracket\setminus\llbracket\alpha\rrbracket}\llbracket\Delta\rrbracket_{v}\llbracket\beta\rrbracket_{v}p(v)}{\sum_{v\notin\llbracket\Delta\rrbracket\setminus\llbracket\alpha\rrbracket}\llbracket\Delta\rrbracket_{v}\llbracket\beta\rrbracket_{v}p(v)}=1.
\end{eqnarray*}
Here, we have excluded all $v\in\llbracket\Delta\rrbracket\setminus\llbracket\alpha\rrbracket$ because of $p(v)=0$. (Cut) Since $\Delta\vapprox_{1}\beta$ holds, $\llbracket\Delta\rrbracket\subseteq\llbracket\beta\rrbracket$ or $p(v)=0$ holds, for all $v\in\llbracket\Delta\rrbracket\setminus\llbracket\beta\rrbracket$. Since $\Delta,\beta\vapprox_{1}\alpha$ holds, $\llbracket\Delta,\beta\rrbracket\subseteq\llbracket\alpha\rrbracket$ or $p(v)=0$ holds, for all $v\in\llbracket\Delta,\beta\rrbracket\setminus\llbracket\alpha\rrbracket$. Let $X=(\llbracket\Delta\rrbracket\setminus\llbracket\beta\rrbracket)\cup(\llbracket\Delta,\beta\rrbracket\setminus\llbracket\alpha\rrbracket)$. For all $v\notin X$, it is thus true that if $v\in\llbracket\Delta\rrbracket$ holds then $v\in\llbracket\beta\rrbracket$ and $v\in\llbracket\alpha\rrbracket$ hold. We thus have
%
\begin{eqnarray*}
p(\alpha|\Delta,\beta)=\frac{p(\alpha,\Delta,\beta)}{p(\Delta,\beta)}=\frac{\sum_{v}\llbracket\alpha\rrbracket_{v}\llbracket\Delta\rrbracket_{v}\llbracket\beta\rrbracket_{v}p(v)}{\sum_{v}\llbracket\Delta\rrbracket_{v}\llbracket\beta\rrbracket_{v}p(v)}=\frac{\sum_{v\notin X}\llbracket\Delta\rrbracket_{v}p(v)}{\sum_{v\notin X}\llbracket\Delta\rrbracket_{v}p(v)}=1.
\end{eqnarray*}
\end{proof}
\par
The next theorem states inferential properties of the Bayesian entailment $\vapprox_{\omega}$ with probability $\omega$ where $0.5<\omega<1$.
\begin{theorem}\label{thrm:3}
Let $\omega$ be a probability where $0.5<\omega<1$. The Bayesian entailment $\vapprox_{\omega}$ satisfies Reflexivity, but does not satisfy Monotony and Cut.
\end{theorem}
\begin{proof}
(Reflexivity) Obvious from $\models~\subseteq~\vapprox_{1}~\subseteq~\vapprox_{\omega}$. (Monotony) We show a counter-example. Given the set $\{a,b\}$ of propositional symbols, consider the probability distribution over valuation functions shown in Table \ref{tab:m}. Note that $\sum_{v}p(v)=1$ holds. It is the case that
\begin{eqnarray*}
&&p(a)=p(v_{3})+p(v_{4})=(1-\omega)+(2\omega-1)=\omega\\
&&p(a|b)=\frac{p(v_{4})}{p(v_{2})+p(v_{4})}=\frac{2\omega-1}{(1-\omega)+(2\omega-1)}=\frac{2\omega-1}{\omega}.
\end{eqnarray*}
$\omega>\frac{2\omega-1}{\omega}$ holds if and only if $(\omega-1)^{2}>0$ holds. It thus always true when $0.5<\omega<1$. Therefore, $\vapprox_{\omega}a$ but $b\not\vapprox_{\omega}a$ hold.
(Cut) We show a counter-example. Consider the probability distribution over valuation functions shown in Table \ref{tab:c}. Note that $\sum_{v}p(v)=1$ holds. It is the case that
\begin{eqnarray*}
&&p(a)=p(v_{3})+p(v_{4})=\omega(1-\omega)+\omega^{2}=\omega\\
&&p(a\land b|a)=\frac{p(v_{4})}{p(v_{3})+p(v_{4})}=\frac{\omega^{2}}{\omega(1-\omega)+\omega^{2}}=\omega\\
&&p(a\land b)=p(v_{4})=\omega^{2}
\end{eqnarray*}
$\omega>\omega^{2}$ always true when $0.5<\omega<1$. Therefore, $\vapprox_{\omega} a$ and $a\vapprox_{\omega} a\land b$ hold, but $\not\vapprox_{\omega} a\land b$ holds.
\end{proof}

\begin{table}[t]
\begin{center}
\begin{tabular}{cc}
\begin{minipage}{0.5\hsize}
\caption{Counter-example of Monotony}
\label{tab:m}
\begin{center}
\begin{tabular}{c|c|cc}
& $p(V)$ & $a$ & $b$\\\hline
$v_{1}$ & $0$ & $0$ & $0$\\
$v_{2}$ & $1-\omega$ & $0$ & $1$\\
$v_{3}$ & $1-\omega$ & $1$ & $0$\\
$v_{4}$ & $2\omega-1$ & $1$ & $1$
\end{tabular}
\end{center}
\end{minipage}
\begin{minipage}{0.5\hsize}
\caption{Counter-example of Cut}
\label{tab:c}
\begin{center}
\begin{tabular}{c|c|cc}
& $p(V)$ & $a$ & $b$\\\hline
$v_{1}$ & $0$ & $0$ & $0$\\
$v_{2}$ & $1-\omega$ & $0$ & $1$\\
$v_{3}$ & $\omega(1-\omega)$ & $1$ & $0$\\
$v_{4}$ & $\omega^{2}$ & $1$ & $1$
\end{tabular}
\end{center}
\end{minipage}
\end{tabular}
\end{center}
\end{table}
\begin{example}
Given $\omega=0.8$ in Table \ref{tab:m}, $p(a)=0.8$ and $p(a|b)=0.75$ hold. Monotony does not hold because $\vapprox_{0.8}a$ holds, but $b\vapprox_{0.8}a$ does not hold. Given $\omega=0.8$ in Table \ref{tab:c}, $p(a)=0.8$, $p(a\land b|a)=0.8$ and $p(a\land b)=0.64$ hold. Cut does not hold because $\vapprox_{0.8}a$ and $a\vapprox_{0.8}a\land b$ hold, but $\vapprox_{0.8}a\land b$ does not hold.
\end{example}
Therefore, in contrast to $\vapprox_{1}$, the Bayesian entailment $\vapprox_{\omega}$ is not a monotonic consequence relation, for all $\omega$ where $0.5<\omega<1$.

\subsection{Non-monotonic Consequence Relation}
%
We next analyze the Bayesian entailment in terms of inferential properties of non-monotonic consequence relations. It is known that there are at leat four core properties characterizing non-monotonic consequence relations:  \textit{Supraclassicality}, \textit{Reflexivity}, \textit{Cautious monotony} and \textit{Cut}. Let $\vsim\subseteq Pow({\cal L})\times{\cal L}$ be a consequence relation on ${\cal L}$. Those properties are formally defined as follows, where $\alpha,\beta\in{\cal L}$ and $\Delta\subseteq{\cal L}$.
%
\begin{itemize}
\item Supraclassicality: $\forall\Delta\forall\alpha$, if $\Delta\vdash\alpha$ then $\Delta\vsim\alpha$
\item Reflexivity: $\forall\Delta\forall\alpha$, $\Delta, \alpha\vsim\alpha$
\item Cautious monotony: $\forall\Delta\forall\alpha\forall\beta$, if $\Delta\vsim\beta$ and $\Delta\vsim\alpha$ then $\Delta,\beta\vsim\alpha$
\item Cut: $\forall\Delta\forall\alpha\forall\beta$, if $\Delta\vsim\beta$ and $\Delta,\beta\vsim\alpha$ then $\Delta\vsim\alpha$
\end{itemize}
We have already discussed reflexivity and cut. Supraclassicality states that the consequence relation extends the monotonic consequence relation. Cautious monotony states that if $\alpha$ is a consequence of $\Delta$ then it is a consequence of supersets of $\Delta$ as well. However, it is weaker than monotony because the supersets are restricted to consequences of $\Delta$. Consequence relations satisfying those properties are often called a \textit{cumulative consequence relation} \cite{brewka:97}.\footnote{This definition is not absolute. The authors \cite{kraus:90} define a cumulative consequence as a relation satisfying Reflexivity, Left logical equivalence, Right weakening, Cut and Cautious monotony.}
%
%
%
%
\begin{theorem}\label{cumulativity}
Let $\omega$ be a probability where $0.5<\omega<1$. The Bayesian entailment $\vapprox_{\omega}$ satisfies Supraclassicality and Reflexivity, but does not satisfy Cautious monotony and Cut.
\end{theorem}
\begin{proof}
(Reflexivity \& Cut) See Theorem \ref{thrm:3}. (Supraclassicality) This is obvious from $\models~\subseteq~\vapprox_{\omega}$. (Cautious monotony) It is enough to show a counter-example. Given set $\{a,b\}$ of atomic propositions, consider again the distribution over valuation functions shown in Table \ref{tab:m}.
%
%
We have 
\begin{eqnarray*}
&&p(a)=p(v_{3})+p(v_{4})=(1-\omega)+(2\omega-1)=\omega\\
&&p(b)=p(v_{2})+p(v_{4})=(1-\omega)+(2\omega-1)=\omega\\
&&p(a|b)=\frac{p(v_{4})}{p(v_{2})+p(v_{4})}=\frac{2\omega-1}{\omega}.
\end{eqnarray*}
$\omega\leq\frac{2\omega-1}{\omega}$ holds if and only if $(\omega-1)^{2}\leq 0$ holds. It is always false when $0.5<\omega<1$ holds.
%
\end{proof}
Theorem \ref{cumulativity} shows that, in general, the Bayesian entailment $\vapprox_{\omega}$ is not cumulative. A natural question here is what inferential properties characterize the Bayesian entailment. We thus introduce two properties: \textit{Classically cautious monotony} and \textit{Classical cut}.
\begin{definition}[Classically cautious monotony and classical cut]
Let $\alpha,\beta\in{\cal L}$ be sentences, $\Delta\in{\cal L}$ be a set of sentences, and $\vsim~\subseteq~Pow({\cal L})\times{\cal L}$ be a consequence relation on ${\cal L}$. Classically cautious monotony and Classical cut are given as follows:
%
\begin{itemize}
\item Classically cautious monotony: $\forall\Delta\forall\alpha\forall\beta$, if $\Delta\vdash\beta$ and $\Delta\vsim\alpha$ then $\Delta,\beta\vsim\alpha$
\item Classical cut: $\forall\Delta\forall\alpha\forall\beta$, if $\Delta\vdash\beta$ and $\Delta,\beta\vsim\alpha$ then $\Delta\vsim\alpha$
\end{itemize}
where $\vdash~\subseteq~Pow({\cal L})\times{\cal L}$ denotes a monotonic consequence relation.
\end{definition}
%
%
Intuitively speaking, Classically cautious monotony and Cut state that only monotonic consequences may be used as premises of the next inference operation. It is in contrast to Cautious monotony and Cut stating that any consequences may be used as premises of the next operation. Classically cautious monotony is weaker than Cautious monotony, and Cautious monotony is weaker than Monotony. Thus, if a consequence relation satisfies Monotony then it satisfies Classically cautious monotony and Cautious monotony as well. We now define a \textit{classically cumulative consequence relation}.
%
%
%
\begin{definition}[Classically cumulative consequence relation]
Let $\vsim~\subseteq~Pow({\cal L})\times{\cal L}$ be a consequence relation on ${\cal L}$. $\vsim$ is said to be a classically cumulative consequence relation if it satisfies all of the following properties.
\begin{itemize}
\item Supraclassicality: $\forall\Delta\forall\alpha$, if $\Delta\vdash\alpha$ then $\Delta\vsim\alpha$
\item Reflexivity: $\forall\Delta\forall\alpha$, $\Delta, \alpha\vsim\alpha$
\item Classically cautious monotony: $\forall\Delta\forall\alpha\forall\beta$, if $\Delta\vdash\beta$ and $\Delta\vsim\alpha$ then $\Delta,\beta\vsim\alpha$
\item Classical cut: $\forall\Delta\forall\alpha\forall\beta$, if $\Delta\vdash\beta$ and $\Delta,\beta\vsim\alpha$ then $\Delta\vsim\alpha$
\end{itemize}
\end{definition}
Any cumulative consequence relation is a classically cumulative consequence relation, but not vice versa. A classically cumulative consequence relation is more conservative in entailment than a cumulative consequence relation. For example, in a cumulative consequence relation $\vsim$, if $(\emptyset,\beta)\in\vsim$ and $(\emptyset,\alpha)\in\vsim$ hold then $(\{\beta\},\alpha)\in\vsim$ necessary holds due to Cautious monotony. However, it does not hold in a classically cumulative consequence relation because $(\emptyset,\beta)\notin\vdash$ might be the case. The same discussion is possible for Cut and Classical cut.

\par
The next theorem shows that Bayesian entailment $\vapprox_{\omega}$ (where $0.5<\omega<1$) is a classically cumulative consequence relation.
%
\begin{theorem}\label{semicumulativity}
Let $\omega$ be a probability where $0.5<\omega<1$. The Bayesian entailment $\vapprox_{\omega}$ is a classically cumulative consequence relation.
\end{theorem}
\begin{proof}
(Supraclassicality \& Reflexivity) See Theorem \ref{cumulativity}. (Classically cautious monotony \& Classical cut) We prove both by showing that $\Delta,\beta\vsim\alpha$ holds if and only if $\Delta\vsim\alpha$ holds, given $\Delta\vdash\beta$ holds. If $p(\Delta)=0$ holds then $\Delta\vdash\beta$, $\Delta\vsim\alpha$ and $\Delta,\beta\vsim\alpha$ obviously hold from the definition. If $p(\Delta)\neq 0$ then we have 
\begin{eqnarray*}
p(\alpha|\Delta,\beta)=\frac{\sum_{v}\llbracket\alpha\rrbracket_{v}\llbracket\Delta\rrbracket_{v}\llbracket\beta\rrbracket_{v}p(v)}{\sum_{v}\llbracket\Delta\rrbracket_{v}\llbracket\beta\rrbracket_{v}p(v)}=\frac{\sum_{v}\llbracket\alpha\rrbracket_{v}\llbracket\Delta\rrbracket_{v}p(v)}{\sum_{v}\llbracket\Delta\rrbracket_{v}p(v)}=p(\alpha|\Delta).
\end{eqnarray*}
We here used the facts $\llbracket\Delta\rrbracket\subseteq\llbracket\beta\rrbracket$ and $p(\Delta,\beta)\neq 0$. $\llbracket\Delta\rrbracket\subseteq\llbracket\beta\rrbracket$ is true because $\Delta\vdash\beta$. $p(\Delta,\beta)\neq 0$ is true because $p(\Delta)\neq0$ and $\llbracket\Delta\rrbracket\subseteq\llbracket\beta\rrbracket$.
\end{proof}
%
Note that it makes no sense to show that any classically cumulative consequence relation is the Bayesian entailment. It corresponds to show that any monotonic consequence relation is the propositional entailment. The classically cumulative consequence relation is a meta-theory used to characterize various logical systems with specific logical language, syntax and semantics.
%
%
%
\begin{example}\label{ce1}
Given $\omega=0.8$ in Table \ref{tab:m}, $p(a\lor b)=1$, $p(a)=0.8$ and $p(a|a\lor b)=0.8$. It is thus the case that $\vdash a\lor b$, $\vapprox_{0.8}a$ and $a\lor b\vapprox_{0.8}a$.
\end{example}
%
\subsection{Preferential Entailment}
The preferential entailment \cite{shoham:87} is a representative approach to a non-monotonic consequence relation. We show that the preferential entailment coincides with a maximum a posteriori entailment, which is an approximation of the Bayesian entailment. The preferential entailment is defined on a preferential structure (or preferential model) $({\cal V},\succ)$, where ${\cal V}$ is a set of valuation functions and $\succ$ is an irreflexive and transitive relation on ${\cal V}$. $v_{1}\succ v_{2}$ represents that $v_{1}$ is preferable\footnote{For the sake of simplicity, we do not adopt the common practice in logic that $v_{2}\succ v_{1}$ denotes $v_{1}$ is preferable to $v_{2}$.} to $v_{2}$ in the sense that the world identified by $v_{1}$ is more normal/typical/natural than the one identified by $v_{2}$. Given preferential structure $({\cal V},\succ)$, $\alpha$ is a preferential consequence of $\Delta$, denoted by $\Delta\vsim_{({\cal V},\succ)}\alpha$, if $\alpha$ is true in all $\succ$-maximal\footnote{$\succ$ has to be smooth (or stuttered) \cite{kraus:90} so that a maximal model certainly exists. That is, for all valuations $v$, either $v$ is $\succ$-maximal or there is a $\succ$-maximal valuation $v' $such that $v'\succ v$.} models of $\Delta$. A consequence relation $\vsim~\subseteq~Pow({\cal L})\times{\cal L}$ is said to be \textit{preferential} \cite{kraus:90} if it satisfies the following Or property, as well as Reflexivity, Cut and Cautious monotony, where $\Delta\subseteq{\cal L}$ and $\alpha,\beta,\gamma\in{\cal L}$.
\begin{itemize}
\item Or: $\forall\Delta\forall\alpha\forall\beta\forall\gamma$, if $\Delta,\alpha\vsim\gamma$ and $\Delta,\beta\vsim\gamma$ then $\Delta,\alpha\lor\beta\vsim\gamma$
\end{itemize}
\par
In line with the fact that maximum a posterior estimation is an approximation of Bayesian estimation, we define a maximum a posteriori (MAP) entailment, denoted by $\vapprox_{MAP}$, that is an approximation of the Bayesian entailment. Several concepts need to be introduced for that. $v_{MAP}$ is said to be a maximum a posteriori estimate if it satisfies
\begin{eqnarray*}
v_{MAP}=\argmax_{v}p(v|\Delta).
\end{eqnarray*}
We now assume that the distribution $p(V|\Delta)$ has a unique peak close to 1 at a valuation function. It means that there is a single state of the world that is very normal/natural/typical. It results in
\begin{eqnarray*}
p(V|\Delta)\simeq
\begin{cases}
1& \it{if}~V=v_{MAP}\\
0& otherwise,
\end{cases}
\end{eqnarray*}
where $\simeq$ denotes an approximation. Note that no one accepting MAP estimation can refuse this assumption in terms of the Bayesian perspective. Now, we have
\begin{eqnarray*}
p(\alpha|\Delta)=\sum_{v}p(\alpha|v)p(v|\Delta)\simeq\sum_{v}p(\alpha|v)\delta(v=v_{MAP})=p(\alpha|v_{MAP}),
\end{eqnarray*}
where $\delta$ is the Kronecker delta that returns $1$ if $v=v_{MAP}$ and otherwise $0$. It is the case that $p(\alpha|v_{MAP})\in\{0,1\}$. Thus, a formal definition of the maximum a posteriori entailment is given as follows.
\begin{definition}[Maximum a posteriori entailment]\label{def:MAPE}
Let $\alpha\in{\cal L}$ be a sentence and $\Delta\subseteq{\cal L}$ be a set of sentences. $\alpha$ is a maximum a posteriori entailment of $\Delta$, denoted by $\Delta\vapprox_{MAP}\alpha$, if $p(\alpha|v_{MAP})=1$ or $p(\Delta)=0$, where $v_{MAP}=\argmax_{v}p(v|\Delta)$.
\end{definition}
\par
Given two ordered sets $(S_{1},\leq_{1})$ and $(S_{2},\leq_{2})$, a function $f$ is said to be an order-preserving (or isotone) map of $(S_{1},\leq_{1})$ into $(S_{2},\leq_{2})$ if $x\leq_{1} y$ implies that $f(x)\leq_{2} f(y)$, for all $x,y\in S_{1}$. The next theorem relates the maximum a posteriori entailment to the preferential entailment.
%
%

\begin{theorem}\label{thrm:MAP1}
Let $({\cal V},\succ)$ be a preferential structure and $p:{\cal V}\rightarrow[0,1]$ be a probability mass function over $V$. If $p$ is an order-preserving map of $({\cal V},\succ)$ into $([0,1]\geq)$ then $\Delta\vsim_{({\cal V},\succ)}\alpha$ implies $\Delta\vapprox_{MAP}\alpha$.
%
%
\end{theorem}
\begin{proof}
It is obviously true when $\Delta$ has no model. Let $v^{*}$ be a $\succ$-maximal model of $\Delta$. It is sufficient to show $p(\alpha|v^{*})=1$ and $v^{*}=\argmax_{v}p(v|\Delta)$. Since $\Delta\vsim_{({\cal V},\succ)}\alpha$, $\alpha$ is true in $v^{*}$. Thus, $p(\alpha|v^{*})=\llbracket\alpha\rrbracket_{v^{*}}=1$. We have
\begin{eqnarray*}
\argmax_{v}p(v|\Delta)=\argmax_{v}p(\Delta|v)p(v)=\argmax_{v}\llbracket\Delta\rrbracket_{v}p(v)
\end{eqnarray*}
Since $\Delta$ is true in $v^{*}$, $\llbracket\Delta\rrbracket_{v^{*}}=1$. Since $p$ is order-preserving, if $v_{1}\succ v_{2}$ then $p(v_{1})\geq p(v_{2})$, for all $v_{1},v_{2}$. Thus, if $v$ is $\succ$-maximal then $p(v)$ is maximal. Therefore, $v^{*}=\argmax_{v}\llbracket\Delta\rrbracket_{v}p(v)$ holds.
\end{proof}
It is the case that $\vsim_{({\cal V},\succ)},\vapprox_{MAP}\subseteq Pow({\cal L})\times{\cal L}$. Theorem \ref{thrm:MAP1} states that $\vsim_{({\cal V},\succ)}~\subseteq~\vapprox_{MAP}$ holds given an appropriate probability distribution over valuation functions.
%
\begin{example}
Suppose the probability distribution over valuation functions shown in the table in Figure \ref{fig:tab} and the preferential structure $({\cal V},\succ)$ given as follows.
\begin{eqnarray*}
{\cal V}&=&\{v_{1},v_{2},v_{3},v_{4}\}\\
\succ&=&\{(v_{1},v_{2}),(v_{1},v_{3}),(v_{1},v_{4}),(v_{3},v_{2}),(v_{4},v_{2})\}
\end{eqnarray*}
The transitivity of $\succ$ is depicted in the graph shown in Figure \ref{fig:tab}. As shown in the graph, the probability mass function $p$ is an order-preserving map of $({\cal V},\succ)$ into $([0,1],\geq)$.
\par
Now, $\{a\lor\lnot b\}\vsim_{({\cal V},\succ)}\lnot b$ holds because $\lnot b$ is true in all $\succ$-maximal models of $\{a\lor\lnot b\}$. Indeed, $\lnot b$ is true in $v_{1}$, which is uniquely $\succ$-maximal in $\{v_{1},v_{3},v_{4}\}$, the models of $a\lor\lnot b$. Meanwhile, $\{a\lor\lnot b\}\vapprox_{MAP}\lnot b$ holds because $p(\lnot b|v_{1})=1$ holds where $v_{1}=\argmax_{v}p(v|a\lor\lnot b)$.
\par
However, $\{a\}\not\vsim_{({\cal V},\succ)}\lnot b$ holds because $\lnot b$ is false in $v_{4}$, which is $\succ$-maximal in $\{v_{3}, v_{4}\}$, the maximal models of $a\lor\lnot b$. In contrast, $\{a\}\vapprox_{MAP}\lnot b$ holds because $p(\lnot b|v_{3})=1$ holds where $v_{3}=\argmax_{v}p(v|a)$.
\end{example}
\begin{figure}[t]
\begin{center}
 \includegraphics[scale=0.4]{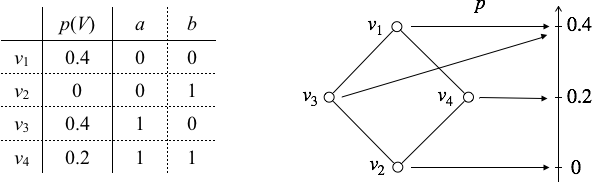}
  \caption{The left table shows the probability distribution over valuation functions. The right graph shows the order-preserving probability mass function that maps each valuation function to its probability.}
  \label{fig:tab}
\end{center}
\end{figure}
The equivalence relation between the maximum a posteriori entailment and the preferential entailment is obtained by restricting the preferential structure to a total order.
%
\begin{theorem}\label{thrm:MAP2}
Let $({\cal V},\succ)$ be a totally-ordered preferential structure and $p:{\cal V}\rightarrow[0,1]$ be a probability mass function over $V$. If $p$ is an order-preserving map of $({\cal V},\succ)$ into $([0,1]\geq)$ then $\Delta\vsim_{({\cal V},\succ)}\alpha$ if and only if $\Delta\vapprox_{MAP}\alpha$.
\end{theorem}
\begin{proof}
Same as Theorem \ref{thrm:MAP1}. Only difference is that the model $v^{*}$ exists uniquely. Then, only $v^{*}$ satisfies $v^{*}=\argmax_{v}\llbracket\Delta\rrbracket_{v}p(v)$.
%
\end{proof}
%

\section{Discussion and Conclusions}\label{sec:dis}
%
There are a lot of research papers combining logic and probability theory, e.g., \cite{adams:98,Fraassen:81b,Morgan:83a,Morgan:82a,Morgan:82b,Morgan:83b,Cross:93,Leblanc:79,Leblanc:83,Fraassen:83,Pearl:91,Goosens:79,Richardson:06}. Their common interest is not the notion of truth preservation but rather probability preservation, where the uncertainty of the premises preserves the uncertainty of the conclusion. They are different from ours because they presuppose and extend classical logical consequence.
\par
Besides the preferential entailment, various other semantics of non-monotonic consequence relations have been proposed such as plausibility structure \cite{friedman:96}, possibility structure \cite{dubois:90}, ranking structure \cite{goldszmidt:92} and $\varepsilon$-semantics \cite{adams:75,pearl:89}. The common idea of the first three approaches is that $\Delta$ entails $\alpha$ if $\llbracket\Delta\land\alpha\rrbracket$ is preferable to $\llbracket\Delta\land\lnot\alpha\rrbracket$ given a preference on models. The idea of the last approach is that $\Delta$ entails $\alpha$ if $p(\alpha|\Delta)$ is close to one in the dependency network quantifying the strength of the causal relationship between sentences. In contrast to the last approach, we focus on the causal relationship between sentences and models, i.e., states of the world. Any sentences are conditionally independent given a model. This fact makes it possible to update the probability distribution over models using observed sentences $\Delta$, and then to predict the truth of unobserved sentence $\alpha$ only using the distribution. It is different from the first three approaches assuming a preference on models prior to the analysis. It is also different from the last approach introducing a new tricky semantics, i.e., $\varepsilon$-semantics, to handle the interaction between models and sentences outside their probabilistic inference. The characteristic allows us to answer the open question \cite{brewka:97}:
\begin{quote}
Perhaps, the greatest technical challenge left for circumscription and model preference theories in general is how to encode preferences among abnormalities or defaults.
\end{quote}
The abnormalities or defaults can be seen as unobserved statements. We thus think that the preferences should be encoded by their posterior probabilities derived by taking into account all the uncertainties of models.
\par
A natural criticism against our work is that the Bayesian entailment is inadequate as a non-monotonic consequence relation due to the lack of Cautious monotony and Cut. Indeed, Gabbay \cite{gabbay:85} considers, on the basis of his intuition, that non-monotonic consequence relations satisfy at least Cautious monotony, Reflexivity and Cut. However, it is controversial because of unintuitive behavior of Cautious monotony and Cut in extreme cases. A consequence relation $\vsim$ with Cut, for instance, satisfies $\vsim x_{N+1}$ when it satisfies $\vsim x_{1}$ and $x_{i}\vsim x_{i+1}$, for all $i (1\leq i\leq N)$. This is very unintuitive when $N$ is large. Brewka \cite{brewka:97} in fact points out the infinite transitivity as a weakness of Cut. In this paper, we reconcile both the positions by providing the alternative inferential properties: Classically cautious monotony and Classical cut. The reconciliation does not come from our intuition, but from theoretical analysis of the Bayesian entailment. What we introduced to define the Bayesian entailment is only the probability distribution over valuation functions, representing uncertainty of states of the world. Given the distribution, the Bayesian entailment is simply derived from the laws of probability theory. Furthermore, the preferential entailment satisfying Cautious monotony and Cut is shown to correspond to the maximum a posteriori entailment that is an approximation of the Bayesian entailment. It tells us that Cautious monotony and Cut are ideal under the special condition that a state of the world exists deterministically. They are not ideal under the general perspective that states of the world are probabilistically distributed.
\par
The Bayesian entailment is flexible to extend. For example, a possible extension of Figure \ref{dependency2} is a hidden Markov model shown in Figure \ref{fig:DBN}. It has a valuation variable and a sentence(s) variable, for each time step $t$ where $1\leq t\leq N$. Entailment $\Delta_{1},...,\Delta_{N}\vapprox_{\omega}\alpha_{N}$ defined in accordance with Definition \ref{def:BE} concludes $\alpha_{N}$ by taking into account not only the current observation $\Delta_{N}$ but also the previous states of the world $V_{N-1}$ updated by all of the past observations $\Delta_{1},..., \Delta_{N-1}$. It is especially useful when observations are contradictory, ambiguous or easy to change.
\par
Finally, our hypothesis is that the Bayesian entailment can be a mathematical model of how human brains implement an entailment. Recent studies of neuroscience, e.g., \cite{lee:03,knill:04,george:05,ichisugi:07,chikkerur:10,colombo:12,funamizu:16}, empirically show that Bayesian inference or its approximation explains several functions of the cerebral cortex, the outer portion of the brain in charge of higher-order functions such as perception, memory, emotion and thought. It raises the Bayesian brain hypothesis \cite{friston:12} that the brain is a Bayesian machine. Since logic, as the law of thought, is a product of a human brain, it is natural to think there is a Bayesian interpretation of logic. Of course, we understand that the Bayesian brain hypothesis is controversial and that it is a subject to a scientific experiment. We, however, think that this paper provides sufficient evidences for the hypothesis in terms of logic.
%
%
\par
This paper gives a Bayesian account of entailment and characterizes its abstract inferential properties. The Bayesian entailment was shown to be a monotonic consequence relation in an extreme case. In general, it is a non-monotonic consequence relation satisfying Classical cautious monotony and Classical cut we introduced to reconcile existing conflicting views. The preferential entailment was shown to correspond to a maximum a posteriori entailment, which is an approximation of the Bayesian entailment. We finally discuss merits of our proposals in terms of encoding preferences on defaults, handling change and contradiction, and modeling human entailment.
%
%
\begin{figure}[t]
\begin{center}
 \includegraphics[scale=0.35]{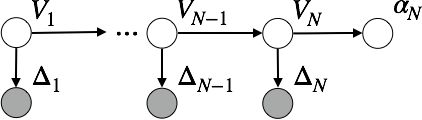}
  \caption{Hidden Markov model for an extended functionality of the Bayesian entailment.}
  \label{fig:DBN}
\end{center}
\end{figure}
%
\bibliographystyle{splncs}
\bibliography{btxkido}
\end{sloppypar}
\end{document}